%% file: fabcont_arxiv.tex
\documentclass{article}
\usepackage{times}
\usepackage{graphicx} 
\usepackage{natbib}
\renewcommand{\cite}{\citep}

\usepackage{algorithm}
\usepackage{algorithmic}

\usepackage{mymacros} 

\usepackage{authblk}
\usepackage{lscape}

\begin{document} 

\title{Rebuilding Factorized Information Criterion:
Asymptotically Accurate Marginal Likelihood}

\author[1,2]{Kohei Hayashi}
\author[3]{Shin-ichi Maeda}
\author[4]{Ryohei Fujimaki}
\affil[1]{Global Research Center for Big Data Mathematics, National Institute of Informatics}
\affil[2]{Kawarabayashi Large Graph Project, ERATO, JST}
\affil[3]{Graduate School of Informatics,
Kyoto University}
\affil[4]{Big Data Analytics, NEC Knowledge Discovery Laboratories}

\maketitle

\input{abstract}
\input{introduction_v3}
\input{problemsetting}
\input{gfic}
\input{gfab}

\input{vbinterpretation_v5}
\begin{landscape}
\input{relatedwork_table}

\end{landscape}
\input{relatedwork}
\input{experiments}

\input{conclusion}

\bibliography{fabcont}
\bibliographystyle{abbrvnat}

\appendix
\input{appendix}

\end{document}

%% file: abstract.tex
\begin{abstract}
Factorized information criterion (FIC) is a recently developed
approximation technique for the marginal log-likelihood, which
provides an automatic model selection framework for a few latent
variable models (LVMs) with tractable inference algorithms.  This
paper reconsiders FIC and fills theoretical gaps of previous FIC
studies. First, we reveal the core idea of FIC that allows 
generalization for a broader class of LVMs, including continuous
LVMs, in contrast to previous FICs, which are applicable only to binary
LVMs. Second, we investigate the model selection mechanism of the
generalized FIC. 
Our analysis provides a formal justification of FIC as a model 
selection criterion for LVMs and also a systematic procedure for pruning
redundant latent variables that have been removed heuristically
in previous studies.  
Third, we provide an interpretation of FIC as a variational 
free energy and uncover a few previously-unknown their relationships.
A demonstrative study on Bayesian principal component analysis
is provided and numerical experiments support our
theoretical results.
\end{abstract}


%% file: introduction_v3.tex
\section{Introduction}\label{sec:introduction}

The marginal log-likelihood is a key concept of Bayesian model
identification of latent variable models~(LVMs), such as mixture models
(MMs), probabilistic principal component analysis, and hidden Markov
models (HMMs).  Determination of dimensionality of latent variables is
an essential task to uncover hidden structures behind the observed
data as well as to mitigate overfitting.  In general, LVMs are
\emph{singular}~(i.e., mapping between parameters and probabilistic
models is not one-to-one) and such classical information criteria
based on the regularity assumption as the Bayesian information criterion
(BIC)~\cite{schwarz78} are no longer justified.  Since exact evaluation
of the marginal log-likelihood is often not available, approximation
techniques have been developed using sampling (i.e., Markov Chain
Monte Carlo methods~(MCMCs)~\cite{hastings70}), a variational lower
bound (i.e., the variational Bayes methods~(VB)~\cite{attias99,jordan99}),
or algebraic geometry (i.e., the widely applicable
BIC~(WBIC)~\cite{watanabe13}).  However, model selection using these
methods typically requires heavy computational cost (e.g., a large
number of MCMC sampling in a high-dimensional space, an outer loop for
VB/WBIC.)

In the last few years, a new approximation technique and an inference
method, factorized information criterion (FIC) and factorized
asymptotic Bayesian inference (FAB), have been developed for some
binary LVMs~\cite{fujimaki12a,fujimaki12b, hayashi13,eto14}.  Unlike
existing methods which evaluate approximated marginal log-likelihoods
calculated for each latent variable dimensionality~(and therefore need
an outer loop for model selection), FAB finds an effective
dimensionality via an EM-style alternating optimization procedure.

For example, let us consider a $K$-component MM 
for $N$ observations $\X^\T=(\x_1,\dots,\x_N)$ with one-of-$K$ coding latent variables
$\Z^\T=(\z_1,\dots,\z_N)$, mixing coefficients $\vbeta=(\beta_1,\dots,\beta_K)$,
and $D_{\Local}$-dimensional component-wise parameters
$\Local=\{\local_1,\dots,\local_K\}$. 
By using Laplace’s method to the marginalization of the
log-likelihood, FIC of MMs~\cite{fujimaki12a} is derived by
\begin{align}\label{eq:FIC-MM}
&\FIC{MM}(K) \equiv
\max_q \E_q\left[\log p(\X,\Z\mid\ML{\vbeta},\ML{\Local},K) \right. \notag
\\
&\qquad\left.-\sum_k\frac{D_{\local_k}}{2}\log\frac{\sum_nz_{nk}}{N}\right] 
+ H(q) - \frac{D_{\Both}}{2}\log N,
\end{align}
where $q$ is the distribution of $\Z$, $\ML{\vbeta}$ and $\ML{\Local}$
are the maximum joint-likelihood estimators (MJLEs)\footnote{%
Note that MJLE is not equivalent to maximum \textit{a posteriori} estimator (MAP). MJLE is given by
$\argmax_\Theta p(\X, \Z | \Theta)$
and the MAP is given by
$\argmax_\Theta p(\X, \Z | \Theta) p(\Theta)$.%
},
$D_{\Both}=D_{\Local} + K - 1$ is the total dimension of $\Local$ and
$\vbeta$, and $H(q)$ is the entropy of $q$.  A key characteristic of
FIC can be observed in the second term of Eq.~\eqref{eq:FIC-MM}, which
gives the penalty in terms of model complexity.  As we can see, the
penalty term decreases when $\sum_n z_{nk}$---the number of effective
samples of the $k$-th component---is small, i.e., $\Z$ is degenerated.
Therefore, through the optimization of $q$, the degenerate dimension
is automatically pruned until a non-degenerated $\Z $ is found.
This mechanism makes FAB a one-pass model selection algorithm and
computationally more attractive than the other methods.  The validity
of the penalty term has been confirmed for other binary LVMs, e.g.,
HMMs~\cite{fujimaki12b}, latent feature models~\cite{hayashi13}, and
mixture of experts~\cite{eto14}.

Despite FAB's practical success compared with BIC and VB, it is
unclear that what conditions are actually necessary to guarantee that
FAB yields the true latent variable dimensionality.  In addition, the
generalization of FIC for non-binary LVMs still remains an
important open issue.  In case that $\Z$ takes negative and/or
continuous values, $\sum_n z_{nk}$ is no longer interpretable as the
number of effective samples, and we loose the clue for finding the
redundant dimension of $\Z$.
%

This paper proposes generalized FIC (gFIC), given by
\begin{align}\label{eq:GFIC}
\gFIC(K) \equiv
&\E_{q^*}[\ajl(\Z, \ML{\Both}, K)] + H(q),
\\
\ajl(\Z, \Both, K) =& \log p(\X,\Z\mid\Both,K) -
 \frac{1}{2}\log|\F_{\Local}| - \frac{D_{\Both}}{2}\log N. \notag
\end{align}
Here, $q^*(\Z)\equiv p(\Z\mid\X,K)$ is the marginal posterior and
$\F_{\ML{\Local}}$ is the Hessian matrix of $-\log p(\X,\Z\mid
\Both,\model)/N$ with respect to $\Local$.  In gFIC, the penalty term
is given by the volume of the (empirical) Fisher information matrix.
It naturally penalizes model complexity even when the latent variable
$\Z$ takes negative and/or continuous values.  Accordingly, gFIC is
applicable to a broader class of LVMs, such as Bayesian principal
component analysis (BPCA) ~\cite{bishop98}.

Furthermore, we prove that FAB automatically prunes redundant
dimensionality along with optimizing $q$, and gFIC for the optimized
$q$ asymptotically converges to the marginal log-likelihood with a
constant order error under some reasonable assumptions. This
justifies gFIC as a model selection criterion for LVMs and further a
natural one-pass model ``pruning'' procedure is derived, which is
performed heuristically in previous FIC studies.  We also provide an
interpretation of gFIC as a variational free energy and uncover
a few previously-unknown their relationships.  
This interpretation gives formal conditions for justifying
that model selection by the VB marginal log-likelihood.
 
Finally, we demonstrate the validity of gFIC by applying it to BPCA.
The experimental results agree with to the theoretical properties of
gFIC.



%% file: problemsetting.tex
\section{LVMs and Degeneration}\label{sec:lvms}

We first define the class of LVMs we deal with in this paper. Here, we
consider LVMs that have $K$-dimensional latent variables $\z_n$
(including the MMs in the previous section), but now $\z_n$ can take
not only binary but also real values. Given $\X$ and a model family
(e.g., MMs), our goal is to determine $K$ and we refer to this as a
\emph{model}. Note that we sometimes omit the notation $K$ for the
sake of brevity, if it is obvious from the context.

The LVMs have $D_{\Local}$-dimensional local parameters
$\Local=\{\local_1,\dots,\local_K\}$ and $D_\Global$-dimensional
global parameters $\Global$, which can include hyperparameters of the
prior of $\Z$.  We abbreviate them as $\Both=\{\Global,\Local\}$ and
assume that the dimension $D_{\Both}=D_{\Global}+D_{\Local}$ is
finite. Then, we define the joint probability: 
$p(\X,\Z,\Both)=p(\X\mid\Z,\Both)p(\Z\mid\Both)p(\Both)$
where $\log p(\X,\Z\mid \Both)$ is twice differentiable at
$\Both\in\pspace$ and let $\F_{\Both} \equiv$
%
\begin{align*}
  \begin{pmatrix}
    \F_{\Global}&\F_{\Global,\Local}\\
    \F_{\Global,\Local}^\T&\F_{\Local}
  \end{pmatrix}
=-
  \begin{pmatrix}
    \frac{\partial}{\partial\Global^\T}\\
    \frac{\partial}{\partial\Local^\T}
  \end{pmatrix}
  \begin{pmatrix}
    \frac{\partial}{\partial\Global}
    \frac{\partial}{\partial\Local}
  \end{pmatrix}
  \frac{\log p(\X,\Z\mid \Both)}{N}.
\end{align*}
Note that the MJLE $\ML{\Both}\equiv\argmax_\Both \log
p(\X,\Z\mid\Both)$ depends on $\Z$ (and $\X$).  In addition, $\log
p(\X,\Z\mid\Both)$ can have multiple maximizers, and $\ML{\Both}$
could be a set of solutions.

Model redundancy is a notable property of LVMs.  Because the
latent variable $\Z$ is unobservable, the pair $(\Z, \Both)$ is not
necessarily determined uniquely for a given $\X$. In other words,
there could be pairs $(\Z, \Both)$ and
$(\tilde{\Z},\tilde{\Both})$, whose likelihoods have
the same value, i.e.,
$p(\X,\Z\mid\Both,\model)=p(\X,\tilde{\Z}\mid\tilde{\Both},\model)$.
Previous FIC studies address this redundancy by introducing a
variational representation that enables treating $\Z$ as fixed, as
we explain in the next section. However, even if $\Z$ is fixed, the redundancy
still remains, namely, the case in which $\Z$ is
``degenerated,'' and there exists an equivalent likelihood with a
smaller model $K'<K$:
\begin{align}\label{eq:equivalence}
  p(\X,\Z\mid\Both,K) = p(\X,\tilde{\Z}_{K'}\mid\tilde{\Both}_{K'}, K').
\end{align}
In this case, $K$ is overcomplete for $\Z$, and $\Z$ lies on the
subspace of the model $K$. As a simple example, let us consider a
three-component MM for which $\Z=(\z,\1-\z,\0)$. In this case,
$\local_3$ is unidentifiable, because the third component is
completely unused, and the $K'=2$-component MM with
$\tilde{\Z}_{2}\equiv(\z,\1-\z)$ and $\tilde{\Both}_2\equiv(\Global,
(\local_1,\local_2))$ satisfies equivalence
relation~\eqref{eq:equivalence}. The notion of degeneration is defined
formally as follows.
\begin{definition}
  Given $\X$ and $K$, $\Z$ is degenerated if there are multiple MJLEs
  and any $\F_{\ML{\Both}}$ of the MJLEs are not positive
  definite. Similarly, $p(\Z)$ is degenerated in distribution, if
  $\E_p[\F_{\ML{\Both}}]$ are not positive definite. Let
  $\kappa(\Z)\equiv\rank(\F_{\ML{\Both}})$ and
  $\kappa(p)\equiv\rank(\E_p[\F_{\ML{\Both}}])$.
\end{definition}
The idea of degeneration is conceptually
understandable as an analogous of linear algebra. Namely, each
component of a model is a ``basis'', $\Z$ are ``coordinates'', and
$\kappa(\Z)$ is the number of necessary components to represent $\X$,
i.e., the ``rank'' of $\X$ in terms of the model family. The
degeneration of $\Z$ is then the same idea of the ``degeneration'' in
linear algebra, i.e., the number of components is too many and $\Both$
is not uniquely determined even if $\Z$ is fixed.

As discussed later, given a degenerated $\Z$ where $\kappa(\Z)=K'$,
finding the equivalent parameters $\tilde{\Z}_{K'}$ and
$\tilde{\Both}_{K'}$ that satisfy Eq.~\eqref{eq:equivalence} is an
important task. In order to analyze this, we assume \assum{switch}:
for any degenerated $\Z$ under a model $K\geq 2$ and $K'<K$, there
exists a continuous onto mapping
$(\Z,\Both)\to(\tilde{\Z}_{K'},\tilde{\Both}_{K'})$ that satisfies
Eq.~\eqref{eq:equivalence}, and $\tilde{\Z}_{K'}$ is not degenerated.
Note that, if $\pspace$ is a subspace of $\R^{D_{\Both}}$, a linear
projection $\V:\R^{D_{\Both}}\mapsto \R^{D_{\Both_{K'}}}$ satisfies
\ref{asm:switch} where $\V$ is the top-$D_{\Both_{K'}}$ eigenvectors
of $\F_{\Both}$. This is verified easily by the fact that, by using
the chain rule, $\F_{\tilde{\Both}_{K'}}=\V\F_{\ML{\Both}}\V^\T$,
which is a diagonal matrix whose elements are positive eigenvalues.
Therefore, $\F_{\tilde{\Both}_{K'}}$ is positive definite and $\tilde{\Z}_{K'}$ is not degenerated.

Let us further introduce a few assumptions required to show the
asymptotic properties of gFIC. Suppose
\assum{nindep} the joint distribution is mutually independent in sample-wise,
\begin{align}
  p(\X, \Z\mid\Both, \model)=\prod_n p(\x_n,\z_n\mid \Both, \model),
\end{align}
and \assum{flat} $\log p(\Both\mid\model)$ is constant, i.e.,
$\lim_{N\to\infty}\log p(\Both\mid\model)/N=0$. 
In addition, \assum{invariant} $p(\Both\mid K)$ is continuous, not
improper, and its support $\pspace$ is compact and the whole space.
Note that for almost all $\Z$, we expect that $\ML{\Both}\in\pspace$
is uniquely determined and $\F_{\ML{\Both}}$ is positive definite,
i.e., \assum{regularity} if $\Z$ is not degenerated, then $\log p(\X,\Z\mid
\Both,\model)$ is concave and $\det|\F_{\ML{\Both}}|<\infty$.
%

\subsection{Examples of the LVM Class} \label{sec:bpca}

The above definition covers a broad class of LVMs. Here, we show that, as examples,
MMs and BPCA are included in that class.
Note that \ref{asm:nindep} does not allow correlation among samples
and analysis of cases with sample correlation (e.g. time series
models) remains as an open problem.

\paragraph{MMs}

In the same notation used in Section~1, the joint likelihood is given
by $p(\X,\Z| \Both) =
\prod_n\prod_k\{\beta_kp_k(\x_n|\local_k)\}^{z_{nk}}$
where $p_k$ is the density of component $k$. If
$\local_1,\dots,\local_K$ have no overlap, $\F_{\Local}$ is the
block-diagonal matrix whose block is given by
$\F_{\local_k}=-\sum_n\nabla\nabla \log p_k(\x_n| \local_k)z_{nk}/N$. This shows that the MM is degenerated, when more than one
column of $\Z$ is filled by zero.  For that case, removing such
columns and corresponding $\local_k$ suffices as
$\tilde{\Z}_{K'}$ and $\tilde{\Both}_{K'}$ in \ref{asm:switch}. Note
that if $p_k$ is an exponential-family distribution
$\exp(\x_n^\T\local_k - \psi(\local_k))$, $-\nabla \nabla \log
p_k(\x_n| \local_k)=\nabla \nabla\psi(\local_k)=\C$ does not depend
on $n$ and $\gFIC{}$ recovers the original formulation of $\FIC{MM}$,
i.e.,
$\frac{1}{2}\log|\F_{\ML{\local_k}}|=\frac{1}{2}\log|\C(\frac{\sum_nz_{nk}}{N})|=\frac{D_{\local_k}}{2}\log\frac{\sum_nz_{nk}}{N}
+ \const$

\paragraph{BPCA}

Suppose $\X\in\R^{N\times D}$ is centerized, i.e., $\sum_n\x_n = \0$.
Then, the joint likelihood of $\X$ and $\Z\in\R^{N\times K}$ is given
by $p(\X,\Z| \Both) = \prod_n N(\x_n|\W\z_n, \frac{1}{\lambda}\I) N(\z_n|\0,\I)$, 
%
where $\Local=\W=(\w_{\cdot 1},\dots,\w_{\cdot K})$ is a linear basis and
$\Global=\lambda$ is the reciprocal of the noise variance. 
Note that the original study of BPCA~\cite{bishop98} introduces the
additional priors $p(\W)=\prod_dN(\w_d|\0,\diag(\valpha^{-1}))$ and
$p(\lambda)=\mathrm{Gamma}(\lambda|a_{\lambda},b_{\lambda})$ and the
hyperprior
$p(\valpha)=\prod_k\mathrm{Gamma}(\alpha_k|a_{\valpha},b_{\valpha})$. In
this paper, however, we do not specify explicit forms of those priors
but just treat them as $O(1)$ term.

Since there is no second-order interaction between $\w_{i}$ and
$\w_{j\not= i}$, the Hessian $\F_{\Local}$ is a block-diagonal and
each block is given by $\frac{\lambda}{N}\Z^\T\Z$. The penalty term is
then given as
\begin{align}
  -\frac{1}{2}\log|\F_{\Local}|=-\frac{D}{2}(K\log\lambda +
  \log|\frac{1}{N}\Z^\T\Z|),
\end{align}
and $\Z$ is degenerated, if $\rank(\Z)<K$.
Suppose that $\Z$ is degenerated, let $K'=\rank(\Z) < K$, and let the
SVD be
$\Z=\U\diag(\vsigma)\V^{\T}=(\U_{K'},\0)\diag(\vsigma_{K'},\0)(\V_{K'},\0)^\T$,
where $\U_{K'}$ and $\V_{K'}$ are $K'$ non-zero singular vectors and
$\vsigma_{K'}$ is $K'$ non-zero singular values. From the definition
of $\F_{\Local}$, the projection $\V$ removes the degeneration of
$\Z$, i.e., by letting $\tilde{\Z}=\Z\V$ and $\tilde{\W}=\W\V$,
\begin{align*}
  \log p(\X, \Z\mid \Both, K)
  &=
  -\frac{\lambda}{2} \fnorm{\X - \Z\W^\T} -\frac{1}{2}\fnorm{\Z} + \const
\\
  &= 
  -\frac{\lambda}{2} \fnorm{\X - \tilde{\Z}\tilde{\W}^\T} -\frac{1}{2}\fnorm{\tilde{\Z}} + \const
\\
  &=
  \log p(\X, \tilde{\Z}_{K'}\mid\{\lambda,\tilde{\W}_{K'}\}, K').
\end{align*}
where $\fnorm{\A}=\sum_{ij}a_{ij}^2$ denotes the Frobenius norm,
$\tilde{\Z}_{K'}=\U_{K'}\diag(\vsigma_{K'})$, and
$\tilde{\W}_{K'}=\W\V_{K'}$. $\V$ transforms $K-K'$ redundant components to
$\0$-column vectors, and we can find the smaller model $K'$ by
removing the $\0$-column vectors from $\tilde{\W}$ and $\tilde{\Z}$, which
satisfies \ref{asm:switch}.

%


%% file: gfic.tex
\section{Derivation of gFIC}

To obtain $p(\X\mid K)$, we need to marginalize out two variables:
$\Z$ and $\Both$. Let us consider the variational form for $\Z$,
written as
  \begin{align}
    \log p(\X| K) =& \E_q[\log p(\X,\Z| K)] + H(q) +\KL{q}{q^*} \label{eq:variational-KL}
\\
=&\E_{q^*}[\log p(\X,\Z| K)] + H(q^{*}),\label{eq:Free energy at q*}
  \end{align}
where $\KL{q}{p}=\int q(x)\log q(x)/p(x)\d x$ is the
Kullback-Leibler (KL) divergence. 

Variational representation~\eqref{eq:Free energy at q*} allows us to
consider the cases of whether $\Z$ is degenerated or not
separately. In particular, when $\Z\sim q^*(\Z)$ is not degenerated,
then \ref{asm:regularity} guarantees that $p(\X,\Z\mid K)$ is regular,
and standard asymptotic results such as Laplace's method are
applicable. In contrast, if $q^*(\Z)$ is degenerated, $p(\X,\Z\mid K)$
becomes singular and its asymptotic behavior is unclear.

In this section, we analyze the asymptotic behavior of the
variational representation~\eqref{eq:Free energy at q*} in both cases and
show that gFIC is accurate even if $q^*(\Z)$ is degenerated. Our main
contribution is the following theorem.\footnote{A formal proof is
  given in supplemental material.}
\begin{theorem}\label{thm:consistency}
  Let $K'=\kappa(p(\Z\mid
  \X,K))$. Then,
  \begin{align}\label{eq:GFIC-consterror}
    \log p(\X\mid K) = \gFIC(K') + O(1).
  \end{align}
\end{theorem}
We emphasize that the above theorem holds even if the model family does not
include the true distribution of $\X$. To prove Theorem~\ref{thm:consistency}, we first
investigate the asymptotic behavior of $\log p(\X\mid K)$ for the
non-degenerated case.

\subsection{Non-degenerated Cases}


Suppose $K$ is fixed, and consider the marginalization $p(\X,\Z) =
\int p(\X,\Z| \Both)p(\Both) \d\Both$.
%
If $p(\Z|\X)$ is not degenerated, then $\Z\sim p(\Z|\X)$ is not
degenerated with probability one. This suffices to guarantee the regularity 
condition~(\ref{asm:regularity}) and hence to justify the application of Laplace's method, which
approximates $p(\X,\Z)$ in an asymptotic
manner~\cite{tierney86}.
\begin{lemma}\label{lem:laplace} 
  If $\Z$ is not degenerated, $p(\X,\Z)=$
\begin{align}\label{eq:laplace-original}
p(\X,\Z,\ML{\Both})
  |\F_{\ML{\Both}}|^{-1/2}\left(\frac{2\pi}{N}\right)^{D_\Both/2}(1+O(N^{-1})).
\end{align}
\end{lemma}
This result immediately yields the following relation:
    \begin{align}\label{eq:laplace}
      \log p(\X,\Z)= \ajl(\Z,\ML{\Both},K) + O(1).
    \end{align}
Substitution of Eq.~\eqref{eq:laplace} into Eq.~\eqref{eq:Free energy at
  q*} yields  Eq.~\eqref{eq:GFIC-consterror}.
Note that we drop the $O(1)$ terms: $\log p(\ML{\Both})$
(see~\ref{asm:flat}), $\frac{D_{\Both}}{2}\log 2\pi$, and a term
related to $\F_{\ML{\Global}}$ to obtain Eq.~\eqref{eq:laplace}. 
We emphasize here that the magnitude of
$\F_{\Both}$ (and $\F_{\Global}$ and $\F_{\Local}$) is constant by
definition. Therefore, ignoring all of the information of
$\F_{\ML{\Both}}$ in Eq.~\eqref{eq:laplace-original} just gives
another $O(1)$ error and equivalence of
gFIC~\eqref{eq:GFIC-consterror} still holds. However, $\F_{\Local}$
contains important information of which component is effectively used
to captures $\X$. Therefore, we use the relation
$\log|\F_{\Both}|=\log|\F_{\Local}| +
\log|\F_{\Local,\Both}\F_{\Both}^{-1}\F_{\Local,\Both}^{\T}|$ and
remain the first term in gFIC. In Section~\ref{sec:how-model-pruning},
we interpret the effect of $\F_{\Local}$ in more detail.

%
%
%
%


\subsection{Degenerated Cases}\label{sec:degenerated-case}

If $p(\Z\mid\X, K)$ is degenerated, then the regularity condition does
not hold, and we cannot use Laplace's method (Lemma~\ref{lem:laplace})
directly. In that case, however, \ref{asm:switch} guarantees the
existence of a variable transformation
$(\Z,\Both)\to(\tilde{\Z}_{K'},\tilde{\Both}_{K'})$ that replaces the
joint likelihood by the equivalent yet smaller ``regular'' model: $p(\X,\Z\mid K)=$
\begin{align}
& \int p(\X,\Z\mid\Both,K) p(\Both\mid K)\d\Both \nonumber \\ 
=& \int p(\X,\tilde{\Z}_{K'}\mid\tilde{\Both}_{K'},K') \tilde{p}(\tilde{\Both}_{K'}\mid K')\d\tilde{\Both}_{K'}\label{eq:joint-transform}.
\end{align}
Since $\tilde{\Z}_{K'}$ is not degenerated in the model $K'$, we can
apply Laplace's method and obtain asymptotic
approximation~\eqref{eq:laplace} by replacing $K$ by $K'$.
Note that the transformed prior $\tilde{p}(\Both_{K'}\mid K')$ would
differ from the original prior $p(\Both_{K'}\mid K')$. However, since
the prior does not depend on $N$ (\ref{asm:flat}), the difference is at most $O(1)$, which is asymptotically ignorable.

Eq.~\eqref{eq:joint-transform} also gives us an asymptotic form of
the marginal posterior.
\begin{proposition}\label{pro:marginal-posterior}
  \begin{align}\label{eq:marginal-posterior} 
    &p(\Z\mid\X,K) = p_K(\Z)(1+O(N^{-1})),\\
&p_K(\Z) \equiv
    \begin{cases}
      \frac{p(\Z,\X\mid\ML{\Both},K)|\F_{\ML{\Both}}|^{-1/2}}{C} & K = \kappa(\Z),
\\
      p_{\kappa(\Z)}(\vT_{\kappa(\Z)}(\Z)) & K > \kappa(\Z),
    \end{cases}
  \end{align}

  where $\vT_{K'}: \Z\to\tilde{\Z}_{K'}$ as
  Eq.~\eqref{eq:equivalence} and $C$ is the normalizing constant.
\end{proposition}
The above proposition indicates that, if $\kappa(p(\Z\mid\X,K))=K'$,
$p(\Z\mid\X,K)$ is represented by the non-degenerated distribution
$p(\Z\mid\X,K')$. Now, we see that the joint
likelihood~\eqref{eq:joint-transform} and the marginal
posterior~\eqref{eq:marginal-posterior} depend on $K'$ rather than $K$.
Therefore, putting these results into variational
bound~\eqref{eq:Free energy at q*} leads to \eqref{eq:GFIC-consterror}, i.e.,
$\log p(\X\mid K)$ is represented by gFIC of the ``true'' model $K'$.
%

Theorem~\ref{thm:consistency} indicates that, if the model $K$ is
overcomplete for the true model $K'$, $\ln p(\X\mid K)$ takes the
same value as $\ln p(\X\mid K')$.
\begin{corollary}\label{cor:samevalue} For every $K> K'=\kappa(p(\Z\mid \X))$, 
  \begin{align}
    \log p(\X\mid K) = \log p(\X\mid K') + O(1).
  \end{align}
\end{corollary}
This implication is fairly intuitive in the sense that if $\X$
concentrates on the subspace of the model, then marginalization with
respect to the parameters outside of the subspace contributes nothing
to $\ln p(\X\mid K)$. Corollary~\ref{cor:samevalue} also justifies
model selection of the LVMs on the basis of the marginal
likelihood. According to Corollary~\ref{cor:samevalue}, at
$N\to\infty$ redundant models always take the same value of the
marginal likelihood as that of the true model, and we can safely
exclude them from model candidates.


%% file: gfab.tex
\section{The gFAB Inference}

To evaluate gFIC~\eqref{eq:GFIC}, we need to solve several estimation
problems. First, we need to estimate $p(\Z\mid\X,K)$ to minimize the
KL divergence in Eq.~\eqref{eq:variational-KL}. In addition, since
$\ln p(\X\mid K)$ depends on the true model $K'$
(Theorem~\ref{thm:consistency}), we need to check whether the current
model is degenerated or not, and if it is degenerated, we need to
estimate $K'$. This is paradoxical, because we would like to determine
$K'$ through model selection. However, by using the properties of
gFIC, we can obtain $K'$ efficiently by optimization.

\subsection{Computation of gFIC}

By applying Laplace's method to Eq.~\eqref{eq:joint-transform} and
substituting it into the variational form~\eqref{eq:variational-KL}, we
obtain $\log p(\X\mid K) = $
\begin{align}\label{eq:equilibrium}
  \E_q[\ajl(\Z, \ML{\Both},\kappa(q))] + H(q) + \KL{q}{q^*} + O(1).
\end{align}
Since the KL divergence is non-negative, substituting this into
Eq.~\eqref{eq:GFIC-consterror} and ignoring the KL divergence gives a lower bound of
$\gFIC(K')$, i.e.,
\begin{align}\label{eq:objective}
  \gFIC(K') &\geq  \E_q[\ajl(\Z, \ML{\Both},\kappa(q))] + H(q).
\end{align}
This formulation allows us to estimate $\gFIC(K')$ via maximizing the
lower bound.
Moreover, we no longer need to know $K'$---if the initial dimension of
$q$ is greater than $K'$, the maximum of lower bound~\eqref{eq:objective}
attains $\gFIC(K')$ and thus $\log p(\X\mid K')$.
Similarly to other variational inference problems, this optimization
is solved by iterative maximization of $q$ and $\Both$.

\subsubsection{Update of $q$}

As suggested in Eq.~\eqref{eq:equilibrium}, the maximizer of
lower bound~\eqref{eq:objective} is $p(\Z\mid \X)$ in which the
asymptotic form is shown in Proposition~\ref{pro:marginal-posterior}.
Unfortunately, we cannot use this as $q$, because the normalizing
constant is intractable.
One helpful tool is the mean-field approximation of $q$, i.e.,
$q(\Z)=\prod_n q_n(\z_n)$. Although the asymptotic marginal
posterior~\eqref{eq:marginal-posterior} depends on $n$ due to
$\F_{\Both}$, this dependency eventually vanishes for $N\to\infty$,
and the mean-field approximation still maintains the asymptotic
consistency of gFIC.
\begin{proposition}\label{pro:independence}
  Suppose $p(\Z|\X,K)$ is not degenerated in distribution. Then,
  $p(\Z|\X,K)$ converges to $p(\Z|\X,\ML{\Both},K)$, and
  $p(\Z|\X,K)$ is asymptotically mutually independent for
  $\z_1,\dots,\z_n$.
\end{proposition}
In some models, such as MMs~\cite{fujimaki12a}, the mean-field
approximation suffices to solve the variational problem. If it is
still intractable, other approximations are necessary. For example, we
restrict $q$ as the Gaussian density
$q(\Z)=\prod_nN(\z_n|\vmu_n,\vSigma_n)$ for BPCA which we use in the
experiments (Section~\ref{sec:experiments}).

\subsubsection{Update of $\Both$}

After obtaining $q$, we need to estimate $\ML{\Both}$ for each sample
$\Z\sim q(\Z)$, which is also intractable.  Alternatively, we estimate
the expected MJLE $\EML{\Both}=\argmax_{\Both}\E_q[\log
p(\X,\Z\mid\Both)]$. Since the max operator has convexity, Jensen's
inequality shows that replacing $\ML{\Both}$ by $\EML{\Both}$
introduces the following lower bound.
\begin{align*}
&\E_q[\log p(\X,\Z\mid\ML{\Both})] = \E_q[\max_{\Both}\log p(\X,\Z\mid\Both)]
\\
\geq &
\E_q[\log p(\X,\Z\mid\EML{\Both})] = \max_{\Both}\E_q[\log p(\X,\Z\mid\Both)].
\end{align*}
Since $\EML{\Both}$ depends only on $q$, we now need to compute
the parameter only once. Remarkably, $\ML{\Both}$ is consistent with
$\EML{\Both}$ and the above equality holds asymptotically.
\begin{proposition}\label{pro:both-consistency}
  If $q(\Z)$ is not degenerated in distribution, then
  $\ML{\Both}\convp\EML{\Both}$.
\end{proposition}
Since $\E_q[\log p(\X,\Z\mid\Both)]$ is the average of the concave
function~(\ref{asm:regularity}), $\E_q[\log p(\X,\Z\mid\Both)]$ itself
is also concave and the estimation of $\EML{\Both}$ is relatively
easy. If the expectations $\E_q[\log p(\X,\Z|\Both)]$ and
$\E_q[\F_{\Both}]$ are analytically written, then gradient-based
optimization suffices for the estimation. If these is no analytic
form, then stochastic optimization, such as stochastic gradient
assuming $\Z\sim q(\Z)$ as a sample~\cite{kingma13}, might help.
%

\subsubsection{Model Pruning}

During the optimization of $q$, it can become degenerated or nearly
degenerated.  In such a case, by definition of
objective~\eqref{eq:objective}, we need to change the form of
$\ajl(\Z, \ML{\Both}, K)$ to $\ajl(\Z, \ML{\Both},K')$. This can be
accomplished by using the transformation
$(\Z,\Both)\to(\tilde{\Z}_{K'},\tilde{\Both}_{K'})$ and decreasing the
current model from $K$ to $K'$, i.e., removing degenerated
components. We refer to this operation as ``model pruning''.  We
practically verify the degeneration by the rank of $\F_{\Local}$,
i.e., we perform model pruning if the eigenvalues are less than some threshold.

\subsection{The gFAB Algorithm}

\begin{algorithm}[tb]
   \caption{The gFAB algorithm}
   \label{alg:gFAB}
\begin{algorithmic}
   \STATE {\bfseries Input:} data $\X$, initial model $K$, threshold $\delta$
   \REPEAT
   \STATE $q\gets \argmax_{q\in\qspace} \E_q[\ajl(\Z,\EML{\Both},\kappa(q))] + H(q)$
   \IF{$\sigma_K(\F_{\Local})\leq\cdots\leq\sigma_{K'}(\F_{\Local})\leq\delta$}
   \STATE $K\gets K'$ and $(\Z,\EML{\Both})\gets(\tilde{\Z}_{K'},\tilde{\Both}_{K'})$
   \ENDIF
   \STATE $\EML{\Both}\gets \argmax_\Both \E_q[\log p(\X,\Z\mid\Both, K)]$
   \UNTIL{Convergence}
\end{algorithmic}
\end{algorithm}
Algorithm~\ref{alg:gFAB} summarizes the above procedures, solving
the following optimization problem:
\begin{align}\label{eq:new-objective}
  \max_{q\in\qspace}\E_q[\ajl(\Z,\EML{\Both}(q),\kappa(q))] + H(q),
\end{align}
where $\qspace=\{q(\Z)\mid q(\Z)=\prod_n q_n(\z_n)\}$.
As shown in Propositions~\ref{pro:independence} and
\ref{pro:both-consistency}, the above objective is the lower bound of
Eq.~\eqref{eq:objective} and thus of $\gFIC(K')$, and the equality
holds asymptotically.
\begin{corollary}
  \begin{align}
    \begin{cases}
      \gFIC(K') = \text{~Eq.~\eqref{eq:new-objective}}& \text{for $N\to\infty$},
\\
      \gFIC(K') \geq \text{~Eq.~\eqref{eq:new-objective}}& \text{for a finite $N>0$}.
    \end{cases}
  \end{align}
\end{corollary}
The gFAB algorithm is the block coordinate ascent. Therefore, if
the pruning threshold $\delta$ is sufficiently small, each step monotonically increases
objective~\eqref{eq:new-objective}, and the algorithm stops at critical points.
%

A unique property of the gFAB algorithm is that it estimates the true
model $K'$ along with the updates of $q$ and $\EML{\Both}$. If $N$ is sufficiently
large and the initial model $K_{\max}$ is larger than $K'$, the
algorithm learns $p_{K'}(\Z)$ as $q$, according to
Proposition~\ref{pro:marginal-posterior}. At the same time, model
pruning removes degenerated $K-K'$ components. Therefore, if the
solutions converge to the global optima, the gFAB algorithm returns $K'$.
%
 
\subsection{How $\F_{\Both}$ Works?}\label{sec:how-model-pruning}


Proposition~\ref{pro:marginal-posterior} shows that if the model is
not degenerated, objective~\eqref{eq:new-objective} is maximized at
$q(\Z)=p_K(\Z)\propto
p(\Z|\X,\ML{\Both})|\F_{\ML{\Both}}|^{-1/2}$, which is the product
of the unmarginalized posterior $p(\Z|\X,\ML{\Both})$ and the
gFIC penalty term $|\F_{\ML{\Both}}|^{-1/2}$. Since
$|\F_{\ML{\Both}}|^{-1/2}$ has a peak where $\Z$ is degenerating, it
changes the shape of $p(\Z|\X,\ML{\Both})$ and increases the
probability that $\Z$ is degenerated. 
Figure~\ref{fig:skew} illustrates how the penalty term affects the
posterior.
\begin{figure}[tb]
  \centering
  \includegraphics[trim=0 0 0 130,clip,width=1\linewidth]{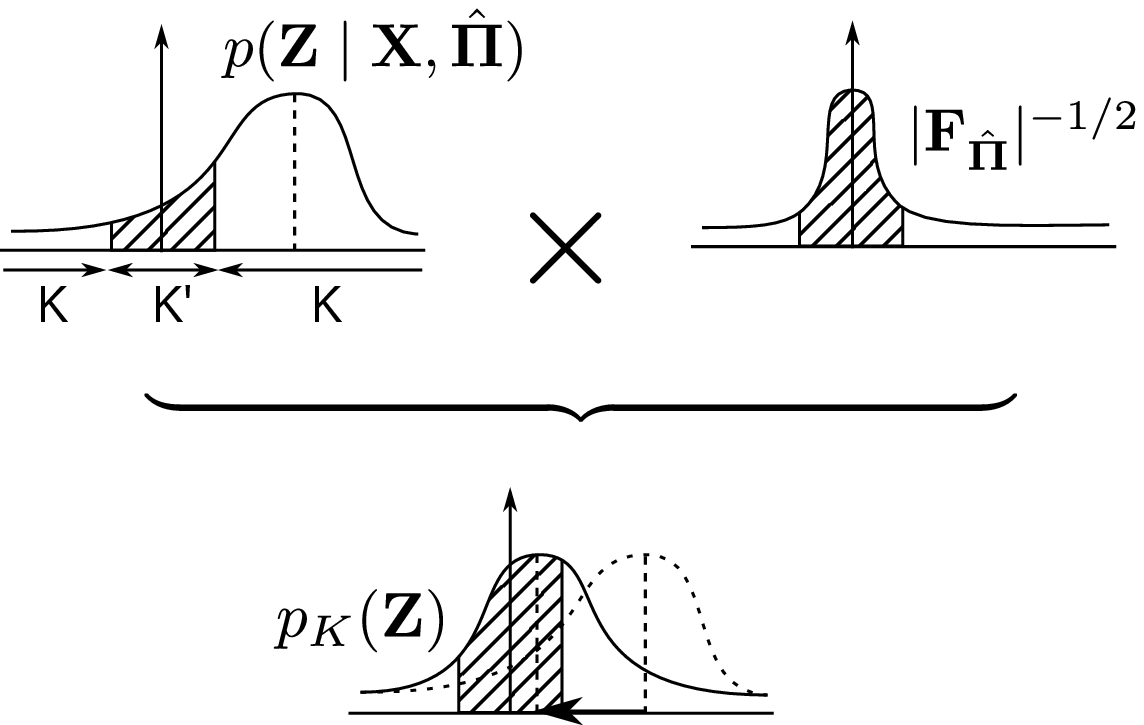}
  \caption{The gFIC penalty $|\F_{\ML{\Local}}|^{-1/2}$ changes the
    shape of the posterior $p(\Z\mid\X,\ML{\Both})$ as increasing the
    probability of degenerated $\Z$ (indicated by diagonal
    stripes).}
  \label{fig:skew}
\end{figure}
%


Note that, if the model family contains the true distribution of $\X$,
then $\F_{\ML{\Both}}$ converges to the Fisher information matrix.
From another viewpoint, $\F_{\ML{\Both}}$ is interpreted as the
covariance matrix of the asymptotic posterior of $\Both$.  As a result
of applying the Bernstein-von Mises theorem, the asymptotic normality
holds for the posterior $p(\Both|\X,\Z)$ in which the covariance is
given by $(N\F_{\ML{\Both}})^{-1}$.
\begin{proposition}\label{pro:normality}
  Let $\vOmega = \sqrt{N}(\Both - \ML{\Both})$. Then, if $\Z$ is not degenerated, 
  $|p(\vOmega\mid\X,\Z) - N(\0, \E[\F_{\ML{\Both}}]^{-1})| \convp 0.$
\end{proposition}
This interpretation has the following implication. In
maximizing the variational lower bound~\eqref{eq:Free energy at q*}, we
maximize $-\frac{1}{2}\log|\F_{\Local}|$. In the gFAB algorithm, this is
equivalent to maximize the posterior covariance and pruning the
components where those covariance diverge to infinity. Divergence of the
posterior covariance means that there is insufficient information to
determine those parameters, which are not necessary for the model and
thus can reasonably be removed.



%% file: vbinterpretation_v5.tex
\section{Relationship with VB}

Similarly to FAB, VB alternatingly optimizes with respect to $\Z$ and
$\Both$, whereas VB treats both of them as distributions. 
%
Suppose $\model \le\model'$, i.e., the case when the posterior $p(\Z
\mid \X, \model')$ is not degenerated in distribution.  Then, the
marginal log-likelihood is written by the variational lower bound: $\log p(\X\mid\model')=$
\begin{align}
&\E_{q(\Z,\Both)}[\log p(\X,\Z,\Both \mid \model')] + H(q(\Z,\Both)) \nonumber \\
&+ \KL{q(\Z, \Both)}{p(\Z, \Both \mid\X,\model')} \nonumber \\
\ge& \E_{q(\Z,\Both)}[\log p(\X,\Z,\Both \mid \model')] + H(q(\Z,\Both))  \nonumber \\
\ge& \E_{q(\Z)q(\Both)}[\log p(\X,\Z,\Both \mid \model')] + H(q(\Z))+H(q(\Both)), \label{eq:fully-vb}
\end{align}
where we use the mean-field approximation $q(\Z,\Both)=q(\Both)q(\Z)$
in the last line.  Minimizing the KL divergence yields the maximizers
of Eq.~\eqref{eq:fully-vb}, given as 
\begin{align}
\tilde q(\Both) &\propto \exp\left( \E_{q(\Z)} \left[\log p(\X,\Z , \Both\mid \model') \right]\right), \label{eq:VB update Pi} 
\\
\tilde q(\Z) &\propto \exp\left( \E_{q(\Both)} \left[\log p(\X,\Z,\Both\mid \model') \right] \right). \label{eq:VB update Z} 
\end{align}

Here, we look inside the optimal distributions to see the relationship
with the gFAB algorithm. Let us consider to restrict the density
$q(\Both)$ to be Gaussian. Since $\E_{q(\Z)} \left[\log p(\X,\Z \mid
  \Both, \model') \right]$ increases proportional to $N$ while $\log
p(\Both)$ does not, $\tilde{q}(\Both)$ attains its maximum around
$\EML{\Both}$. Then, the second order expansion to $\log
\tilde{q}(\Both)$ at $\EML{\Both}$ yields the solution $\tilde
q(\Both) = N(\EML{\Both}, (N\F_{\EML{\Both }})^{-1})$.
We remark that this solution can
be seen as an empirical version of the asymptotic normal posterior
given by Proposition~\ref{pro:normality}.  Then, if we further
approximate $\log p(\X,\Z \mid \Both, \model') $ by the second order
expansion at $\Both = \EML{\Both}$, the other expectation
$\E_{q(\Both)} \left[\log p(\X,\Z \mid \Both, \model') \right]$
appearing in Eq.~\eqref{eq:VB update Z} is evaluated by $\log p(\X,\Z \mid
\EML{\Both}, \model') - \frac{1}{2}\log |\F_{\EML{\Both }}| $.  Under
these approximations, alternating updates of $\{ \EML{\Both },
F_{\EML{\Both }} \}$ and $\tilde q(\Z)$ coincide exactly with the
gFAB algorithm\footnote{Note that model pruning is not
  necessary when $\model \le \model'$.}, which justifies the VB lower
bound as an asymptotic expansion of $p(\X\mid K')$.
\begin{proposition}\label{pro:vbconsistency}
  Let $L_{\mathrm{VB}}(K)$ be the VB lower bound~\eqref{eq:fully-vb}
  with restricting $q(\Both)$ to be Gaussian and approximating the
  expectation in $\log \tilde{q}(\Z)$ by the second order
  expansion. Then, for $\model \le\model'$, $\log p(\X\mid K) = L_{\mathrm{VB}}(K) + O(1).$
\end{proposition}
Proposition~\ref{pro:vbconsistency} states that the VB approximation is
asymptotically accurate as well as gFIC when the model is not
degenerated. For the degenerated case, the asymptotic behavior of
$L_{\mathrm{VB}}(K)$ of general LVMs is unclear; however, a few specific models
such as Gaussian MMs~\cite{watanabe06} and reduced rank
regressions~\cite{watanabe09} have been analyzed in both degenerated
and non-degenerated cases.

Proposition~\ref{pro:vbconsistency} also suggests that the mean-field
approximation does not loose the consistency with $p(\X\mid K)$. As
shown in Proposition~\ref{pro:normality}, for $K\leq K'$, the
posterior covariance is given by $(N\F_{\ML{\Both}})^{-1}$, which goes
to $0$ for $N\to\infty$, i.e., the posterior converges to a
point. Therefore, mutual dependence among $\Z$ and $\Both$ eventually
vanishes in the posterior, and the mean-field assumption holds
asymptotically.
This observation also allows further employment of the mean-field
approximation to $q(\Both)$.  For example, BPCA has two parameters
$\Both=\{\W,\lambda\}$ (see Section~\ref{sec:bpca}), in which the joint
distribution $\tilde q(\W, \lambda)$ has no analytical solution.
However, the independence assumption $q(\W, \lambda) =
q(\W)q(\lambda)$ gives us analytical solutions of $\tilde q(\W )$,
$\tilde q(\lambda)$, and $\tilde q(\Z)$ under suitable conjugate
priors. As discussed above, since both $\W$ and $\lambda$ converge to
points, this approximation still maintains Proposition~\ref{pro:vbconsistency}.


%
%
%


%% file: relatedwork_table.tex
\begin{table*}[tb]
  \centering
  \begin{tabular}{|c||c|c|c|c|c|c|}
\hline
    & EM & BICEM$^\dag$ & VB& CVB& FAB & gFAB$^\dag$ \\\hline\hline
Objective&Eq.~\eqref{eq:obj-EM}&Eq.~\eqref{eq:obj-EM}$-\frac{D_{\Both}}{2}\log N$&Eq.~\eqref{eq:fully-vb}&Eq.~\eqref{eq:Free energy at q*}&\multicolumn{2}{c|}{Eq.~\eqref{eq:objective}}
\\\hline
$\Both$&\multicolumn{2}{c|}{Point estimate}&Posterior w/ MF&Marginalized out&\multicolumn{2}{c|}{Laplace approximation}
\\\hline
$q(\Z)$&\multicolumn{2}{c|}{$=p(\Z\mid\X,\ML{\Both})$}&\multicolumn{2}{c|}{$\simeq p(\Z\mid\X)$}&\multicolumn{2}{c|}{$\propto p(\Z\mid\X)(1+O(1))^\dag$}
\\\hline
$\log p(\X| K\le K')$& $O(\log N)$ & $O(1)^\dag$& \multicolumn{2}{c|}{$O(1)^\dag$} & \multicolumn{2}{c|}{$O(1)^\dag$}
\\\hline
$\log p(\X| K > K')$& \multicolumn{2}{c|}{NA} & \multicolumn{2}{c|}{Generally NA} & \multicolumn{2}{c|}{$O(1)^\dag$}
\\\hline
Applicability&\multicolumn{2}{c|}{Many models}&Many models&Binary LVMs&Binary LVMs&LVMs
\\
\hline
  \end{tabular}
  \caption{A comparison of approximated Bayesian methods. The symbol $\dag$ highlights our contributions. ``MF'' stands for the mean-field approximation. Note that the asymptotic relations with $\log p(\X\mid K)$ hold only for LVMs. }
  \label{tab:comp}
\end{table*}

%% file: relatedwork.tex
\section{Related Work}

\paragraph{The EM Algorithm}\label{sec:em-algorithm}

Algorithm~\ref{alg:gFAB} looks quite 
similar to the EM algorithm, solving
\begin{align}\label{eq:obj-EM}
  \max_{q,\Both}\E_q[\log p(\X\mid\Both, K)] + H(q).
\end{align}
We see that both gFAB and EM algorithms iteratively update the
posterior-like distribution of $\Z$ and estimate $\Both$.  The
essential difference between them is that the EM algorithm infers the
posterior $p(\Z|\X,\ML{\Both})$ in the E-step, but the gFAB algorithm
infers the \emph{marginal} posterior $p(\Z|\X)\simeq
p(\Z|\X,\ML{\Both})|\F_{\ML{\Both}}|^{-1/2}$. As discussed in
Section~\ref{sec:how-model-pruning}, the penalty term
$|\F_{\ML{\Both}}|^{-1/2}$ increases the probability mass of the
posterior, where $\Z$ is degenerating, enabling automatic model
determination through model pruning. 
In contrast, the EM algorithm lacks such pruning mechanism, 
and always overfits to $\X$ as long as $N$ is finite
while $p(\Z|\X)$ eventually
converges to $p(\Z|\X,\ML{\Both})$ for $N\to\infty$ (see
Proposition~\ref{pro:independence}).

Note that Eq.~\eqref{eq:obj-EM} has $O(\log N)$ error against $\log
p(\X)$. Analogously to gFIC, this error is easily reduced to $O(1)$ by
adding $-\frac{D_{\Both}}{2}\log N$. This modification provides
another information criterion, which we refer to as \emph{BICEM}. 

\paragraph{VB Methods}

The relationship between the VB and gFAB algorithms is
discussed in the previous section. 

Collapsed VB (CVB)~\cite{teh06} is a variation of VB. Similarly to FAB,
CVB takes the variational bound after marginalizing out $\Both$ from
the joint likelihood. In contrast to FAB, CVB approximates $q$ in a
non-asymptotic manner, such as the first-order Taylor
expansion~\cite{asuncion09}. Although such approximation has
been found to be accurate in practice, its asymptotic properties, such
as consistency, have not been explored.
Note that as one of those approximations, the mean-field assumption
$q(\Z)\in\qspace$ is used in the original paper on CVB~\cite{teh06},
motivated by the intuition that the dependence among $\{\z_n\}$ is
weak after marginalization. Proposition~\ref{pro:independence}
formally justifies this asymptotic independence assumption on the
marginal distribution employed in CVB.
%
%
%

Several authors have studied about asymptotic behaviors of VB methods
for LVMs. \citet{wang04} investigated the VB approximation for linear
dynamical systems (a.k.a. Kalman filter) and showed the inconsistency
of VB estimation with large observation noise. \citet{watanabe06}
derived an asymptotic variational lower bound of the Gaussian MMs and
demonstrated its usefulness for the model selection.  Recently,
\citet{nakajima14} analyzed the VB learning on latent Dirichlet
allocation (LDA)~\cite{blei03}, who revealed conditions for the consistency and clarified
its transitional behavior of the parameter sparsity.
By comparing with these existing works, we have a contribution
in terms of that our asymptotic analysis is valid for \emph{general} LVMs, rather than
individual models.

\paragraph{BIC and Extensions}

Let $\Y=(\X,\Z)$ be a pair of non-degenerated $\X$ and $\Z$.  By
ignoring all the constant terms of Laplace's
approximation~\eqref{eq:laplace-original}, we obtain
BIC~\cite{schwarz78} considering $\Y$ as an observation, which is
given by the right-hand side of the following equation.
\begin{align*}
  \log p(\Y\mid K) = \log p(\Y\mid \ML{\Both}, K) - \frac{D_{\Both}}{2}\log N + O(1).
\end{align*}
Unfortunately, the above relation does not hold for $p(\X\mid K)$. Since
$p(\X\mid K)=\int p(\Y\mid K)\d\Z$ mixes up degenerated and
non-degenerated cases, $p(\X\mid K)$ always becomes
singular, loosing the condition~\ref{asm:regularity} that
Laplace's approximation holds.

There are several studies that extend BIC to be able to deal with
singular models. 
\citet{watanabe09} evaluates $p(\X\mid K)$ with
an $O(1)$ error for any singular models by using algebraic geometry.
However, it requires an evaluation of the intractable rational number
called the real log canonical threshold.
Recent study~\cite{watanabe13} relaxes this intractable evaluation
to the evaluation of criterion called WBIC at the expense of an $O_p(\sqrt{\log N})$ error.
Yet, the evaluation of WBIC needs an expectation with respect to 
a practically intractable distribution, which usually incurs heavy computation.


%


%% file: experiments.tex
\section{Numerical Experiments}\label{sec:experiments}

We compare the performance of model selection for BPCA explained in
Section~\ref{sec:bpca} with the \texttt{EM} algorithm, \texttt{BICEM}
introduced in Section~\ref{sec:em-algorithm}, simple VB
(\texttt{VB1}), full VB (\texttt{VB2}), and the \texttt{gFAB}
algorithm.
\texttt{VB2} had the priors for $\W,\lambda$, and $\valpha$ 
described in Section~\ref{sec:bpca} in which the hyperparameters were
fixed as $a_{\lambda}=b_{\lambda}=a_{\valpha}=b_{\valpha}=0.01$ by
following~\cite{bishop99}. 
\texttt{VB1} is a simple variant of \texttt{VB2}, which fixed
$\valpha=\1$.
In this experiments, We used the synthetic data $\X=\Z\W^\T + \vE$
where $\W\sim \mathrm{uniform}([0,1])$\footnote{This setting could be
  unfair because \texttt{VB1} and \texttt{VB2} assume the Gaussian
  prior for $\W$. However, we confirmed that data generated by $\W\sim
  N(0,1)$ gave almost the same results.}, $\Z\sim N(\0, \I)$, and
$E_{nd}\sim N(0, \sigma^2)$. Under the data dimensionality $D=30$ and
the true model $K'=10$, we generated data with $N=100,500,1000,$ and
$2000$.
We stopped the algorithms if the relative error was less than
$10^{-5}$ or the number of iterations was greater than $10^4$.
\begin{figure}[tb]
  \centering
  \includegraphics[trim=0 0 0 20,clip,width=.9\linewidth]{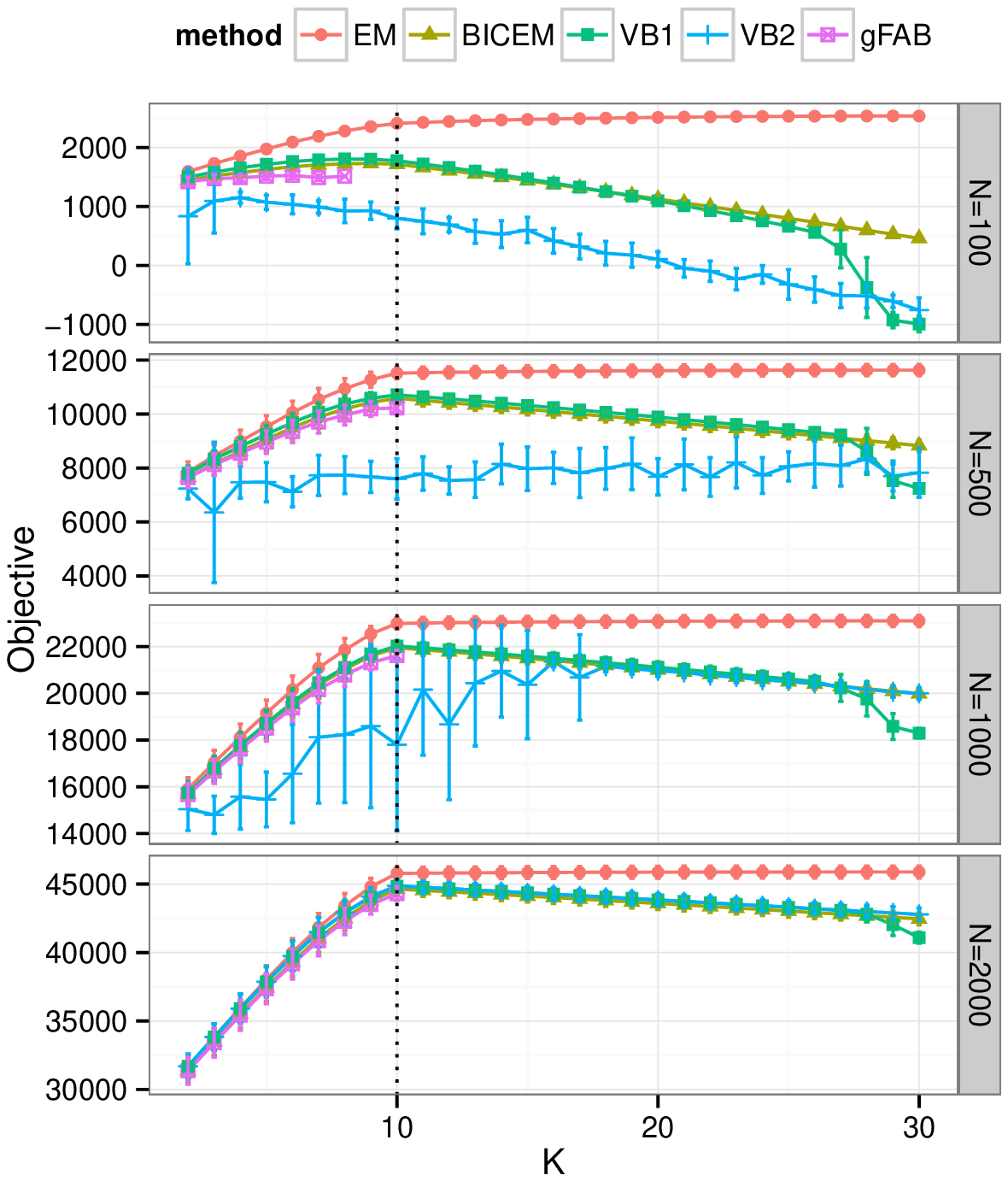}
  \caption{The objective function versus the model $K$. The errorbar
    shows the standard deviations over 10 different random seeds, 
    which affect both data and initial values of the algorithms.}
  \label{fig:ex-toy}
\end{figure}

Figure~\ref{fig:ex-toy} depicts the objective functions after
convergence for $K=2,\dots,30$.
Note that, we performed \texttt{gFAB} with $K=30$ and it finally
converged at $K\simeq 10$ owing to model pruning, which allowed us to
skip the computation for $K\simeq 10,\dots,30$, and the objective
values for those $K$s are not drawn. We see that \texttt{gFAB}
underestimated the model when the number of samples were small ($N<
500$), but it successfully chose $K=10$ with sufficiently large sample
sizes ($N\geq 500$).
In contrast, the objective of \texttt{EM} slightly but monotonically
increased with $K$, which means \texttt{EM} always chose the largest
$K$ as the best model. This is because \texttt{EM} maximizes
Eq.~\eqref{eq:obj-EM}, which does not impose the penalty on
the model complexity brought by the marginalization of $\Both$. 
As our analysis suggested in Section~\ref{sec:em-algorithm},
\texttt{BICEM} and \texttt{VB1} are close to \texttt{gFAB} as $N$
increasing and has a peak around $K'=10$, meaning that \texttt{BICEM}
and \texttt{VB1} are adequate for model selection. However, in
contrast to \texttt{gFAB}, both of them need to compute for all $K$.
Interestingly, \texttt{VB2} were unstable for $N<2000$ and it gave the
inconsistent model selection results. We observed that \texttt{VB2}
had very strong dependence on the initial values. This behavior is
understandable because \texttt{VB2} has the additional prior and
hyperparameters to be estimated, which might produce additional local
minima that make optimization difficult.


%% file: conclusion.tex
\section{Conclusion}

This paper provided an asymptotic analysis for the marginal
log-likelihood of LVMs. As the main contribution, we proposed gFIC for
model selection and showed its consistency with the marginal
log-likelihood. Part of our analysis also provided insight into the EM
and VB methods. Numerical experiments confirmed the validity of our
analysis.

We remark that gFIC is potentially applicable to many other LVMs,
including factor analysis, LDA, canonical correlation analysis, and
partial membership models.  Investigating the behavior of gFIC on
these models is an important future research direction.



%% file: appendix.tex
\section{Proofs}

\begin{proof}[Proof of Proposition~\ref{pro:marginal-posterior}]
  If $\Z$ is not degenerated, then Laplace's method yields
  Eq.~\eqref{eq:laplace}. By collecting from Eq.~\eqref{eq:laplace} the
  terms that depend on $\Z$, we obtain
  \begin{align}
p(\Z\mid\X,K) \propto    p(\Z,\X\mid\ML{\Both},K)|\F_{\ML{\Both}}|^{-1/2}(1+O(N^{-1})).
\end{align}
%

  If $p(\Z\mid\X, K)$ is degenerated, we consider
  the transformation~\eqref{eq:joint-transform}.
Here, the transformed prior $\tilde{p}(\Both_{K'}\mid K')$ would
differ from the original prior $p(\Both_{K'}\mid K')$. However, since
the mapping $\Both\to\tilde{\Both}_{K'}$ is onto~\ref{asm:switch} and
the prior is strictly positive in the whole space of
$\Both$~\ref{asm:invariant}, $\tilde{p}(\Both\mid K')$ is also
strictly positive, including $\ML{\Both}_{K'}=\argmax_{\Both_{K'}}\ln
p(\X,\tilde{\Z}_{K'}\mid\Both_{K'},K')$. Consequently, we can again
use Laplace's method for $\ln
p(\X,\tilde{\Z}_{K'}\mid\ML{\Both}_{K'},K')$, and by collecting the
terms that depend on $\Z$, we obtain
\begin{align}
  p(\X\mid\Z, K) 
  &\propto
  p(\X,\tilde{\Z}_{K'}\mid\ML{\Both}_{K'},K')|\F_{\ML{\Both}_{K'}}|^{-1/2}(1+O(N^{-1}))
\\
  &\propto
  p_{K'}(\tilde{\Z}_{K'},K')(1+O(N^{-1})).
\end{align}
This concludes the proof.
\end{proof}

\begin{proof}[Proof of Theorem~\ref{thm:consistency}]
  First, we prove the case that $p(\Z\mid\X, K)$ is not degenerated.
  In that case, Laplace's approximation yields Eq.~\eqref{eq:laplace}
  in probability, and substituting Eq.~\eqref{eq:laplace} into
  \eqref{eq:Free energy at q*} gives \eqref{eq:GFIC-consterror}.

  If $\kappa(p(\Z\mid\X, K))=K'<K$,
  Proposition~\ref{pro:marginal-posterior} gives us that $p(\Z\mid\X,
  K)=p_{K'}(\Z)(1+O(N^{-1}))$.  Since
  \begin{align*}
    \E_{p(\Z\mid\X, K)}[\log p(\X,\Z\mid K)] = \E_{p_{K'}}[\log
    p(\X,\Z\mid K)] + O(1)
  \end{align*}
 and 
 \begin{align*}
   H(p(\Z\mid\X, K))
   &=(1+O(N^{-1}))H(p_{K'}) + (1+O(N^{-1}))\log(1+O(N^{-1}))
   \\
   &=H(p_{K'})+O(1),
 \end{align*}
 $\log p(\X\mid K)$ is rewritten by
  \begin{align}
 &\E_{p_{K'}}[\log p(\X,\Z\mid K)] + H(p_{K'}) + O(1)
\\
=&\E_{p_{K'}}[\ajl(\ML{\Z}_{K'},\tilde{\Both}_{K'},K')] + H(p_{K'}) 
 + O(1)
  \end{align}
  Here, since the projection
  $\vT_{K'}: \Z\to\tilde{\Z}_{K'}$ is continuous and
  onto~(\ref{asm:switch}), we can describe $p_{K'}(\Z)$ as the density
  of $\Z_{K'}$ by using a change of variables, which we denote by
  $\tilde{p}_{K'}(\Z_{K'})$. Now, we can rewrite the first term as the
  integral over $\Z_{K'}$, i.e.,
\begin{align}
  \E_{p_{K'}}[\ajl(\tilde{\Z}_{K'},\ML{\Both}_{K'},K')]
 =&\int \ajl(\vT_{K'}(\Z),\ML{\Both}_{K'},K')p_{K'}(\vT_{K'}(\Z))\d\Z
\\
 =&\int \ajl(\Z_{K'},\ML{\Both}_{K'},K')\tilde{p}_{K'}(\Z_{K'})\d\Z_{K'}.\label{eq:a2}
\end{align}
Similarly, $\gFIC(K')$ is rewritten using
Proposition~\ref{pro:marginal-posterior} as 
\begin{align}
\gFIC(K')=  \E_{p_{K'}}[\ajl(\Z_{K'}, \ML{\Both}_{K'}, K')] + H(p_{K'}) + O(1)
\end{align}
Again, the first term is written as
\begin{align}
  \E_{p_{K'}}[\mathcal{L}(\Z_{K'}, \ML{\Both}_{K'}, K')]
 &=\int \mathcal{L}(\Z_{K'},\ML{\Both}_{K'}, K')p_{K'}(\vT_{K'}(\Z))\d\Z_{K'}
\\
 &=\int \mathcal{L}(\Z_{K'},\ML{\Both}_{K'}, K')\tilde{p}_{K'}(\Z_{K'})\d\Z_{K'}\label{eq:a3}
\end{align}
Since Eq.~\eqref{eq:a2} and \eqref{eq:a3} are the same, this concludes
Eq.~\eqref{eq:GFIC-consterror}.

\end{proof}

\begin{proof}[Proof of Proposition~\ref{pro:independence}]
  Proposition~\ref{pro:marginal-posterior} shows that, if $\Z$ is
  non-degenerated,
  \begin{align}
    p(\Z\mid \X,K) &\propto p(\X,\Z\mid\ML{\Both})|\F_{\ML{\Both}}|^{-1/2}
\\
             &\propto \prod_n p(\x_n,\z_n\mid\ML{\Both})|\F_{\ML{\Both}}|^{-1/2N}
  \end{align}
  Since $\log|\F_\Both|= O(1)$, $|\F_\Local|^{-1/2N}$ quickly
  diminishes to $1$ for $N\to\infty$. 
\end{proof}

\begin{proof}[Proof of Proposition~\ref{pro:both-consistency}]
  For technical reasons, we redefine the estimators as follows:
  \begin{align}
    \ML{\Both}&\equiv\argmax_{\Both}g_N(\Both)=\argmax_{\Both}\frac{1}{N}\log p(\X,\Z | \Both),
\\
    \EML{\Both}&\equiv\argmax_{\Both}G_N(\Both)=\argmax_{\Both}\E_q[\frac{1}{N}\log p(\X,\Z | \Both)].
  \end{align}
  According to \ref{asm:regularity}, $g_N(\Both)$ is continuous and
  concave, and it uniformly converges to
  $G_N(\Both)$, i.e.,
  \begin{align}
    \sup_{\Both\in\pspace}|g_N(\Both) - G_N(\Both)|\convp 0.
  \end{align}
This suffices to show
  the consistency (for example, see Theorem 5.7 in~\cite{vaart98}.)
\end{proof}
%

%
